\documentclass[english,11pt]{article}
\usepackage[T1]{fontenc}
\usepackage[latin9]{inputenc}
\usepackage{mathrsfs}
\usepackage{amsmath}
\usepackage{amsthm}
\usepackage{amssymb}
\usepackage{graphicx}

\makeatletter
\theoremstyle{plain}
\newtheorem{thm}{\protect\theoremname}
  \theoremstyle{definition}
  \newtheorem{defn}[thm]{\protect\definitionname}
  \theoremstyle{definition}
  \newtheorem{example}[thm]{\protect\examplename}
  \theoremstyle{plain}
  \newtheorem{conjecture}[thm]{\protect\conjecturename}
  \theoremstyle{plain}
  \newtheorem{prop}[thm]{\protect\propositionname}
  \theoremstyle{plain}
  \newtheorem{lem}[thm]{\protect\lemmaname}

\usepackage{times}
\usepackage[margin=1in]{geometry}
\usepackage{url}

\makeatother

\usepackage{babel}
  \providecommand{\conjecturename}{Conjecture}
  \providecommand{\definitionname}{Definition}
  \providecommand{\examplename}{Example}
  \providecommand{\lemmaname}{Lemma}
  \providecommand{\propositionname}{Proposition}
\providecommand{\theoremname}{Theorem}

\begin{document}

\title{How Much Restricted Isometry is Needed In Nonconvex Matrix Recovery?\thanks{32nd Conference on Neural Information Processing Systems (NIPS 2018),
Montréal, Canada.}}

\author{Richard Y.~Zhang\\
University of California, Berkeley\\
\url{ryz@alum.mit.edu}\\
\and Cédric Josz\\
University of California, Berkeley\\
\url{cedric.josz@gmail.com}\\
\and Somayeh Sojoudi\\
University of California, Berkeley\\
\url{sojoudi@berkeley.edu}\\
\and Javad Lavaei\\
University of California, Berkeley\\
\url{lavaei@berkeley.edu}}

\date{~}
\maketitle
\begin{abstract}
When the linear measurements of an instance of low-rank matrix recovery
satisfy a restricted isometry property (RIP)\textemdash i.e. they
are approximately norm-preserving\textemdash the problem is known
to contain \emph{no spurious local minima}, so exact recovery is guaranteed.
In this paper, we show that moderate RIP is not enough to eliminate
spurious local minima, so existing results can only hold for near-perfect
RIP. In fact, counterexamples are ubiquitous: we prove that every
$x$ is the spurious local minimum of a rank-1 instance of matrix
recovery that satisfies RIP. One specific counterexample has RIP constant
$\delta=1/2$, but causes randomly initialized stochastic gradient
descent (SGD) to fail 12\% of the time. SGD is frequently able to
avoid and escape spurious local minima, but this empirical result
shows that it can occasionally be defeated by their existence. Hence,
while exact recovery guarantees will likely require a proof of \emph{no
spurious local minima}, arguments based solely on norm preservation
will only be applicable to a narrow set of nearly-isotropic instances.
\end{abstract}

\section{Introduction}

\global\long\def\R{\mathbb{R}}
\global\long\def\A{\mathbf{A}}
\global\long\def\AA{\mathcal{A}}
\global\long\def\tr{\mathrm{tr}}
\global\long\def\vec{\mathrm{vec}\,}
\global\long\def\rank{\mathrm{rank}\,}
\global\long\def\dom{\mathrm{dom}\,}
\global\long\def\S{\mathbb{S}}
\global\long\def\X{\mathbf{X}}
\global\long\def\XX{\mathcal{X}}
\global\long\def\H{\mathbf{H}}
\global\long\def\HH{\mathcal{H}}
\global\long\def\L{\mathscr{L}}
\global\long\def\M{\mathscr{M}}
\global\long\def\D{\mathcal{D}}
\global\long\def\mat{\mathrm{mat}}
\global\long\def\range{\mathrm{range}}
\global\long\def\cond{\mathrm{cond}}
\global\long\def\ub{\mathrm{ub}}
\global\long\def\lb{\mathrm{lb}}
Recently, several important nonconvex problems in machine learning
have been shown to contain \emph{no spurious local minima}~\cite{ge2015escaping,bhojanapalli2016global,ge2016matrix,boumal2016non,ge2017nospurious,sun2015complete,park2017non}.
These problems are easily solved using local search algorithms despite
their nonconvexity, because every local minimum is also a global minimum,
and every saddle-point has sufficiently negative curvature to allow
escape. Formally, the usual first- and second-order necessary conditions
for local optimality (i.e. zero gradient and a positive semidefinite
Hessian) are also \emph{sufficient} for global optimality; satisfying
them to $\epsilon$-accuracy will yield a point within an $\epsilon$-neighborhood
of a globally optimal solution. 

Many of the best-understood nonconvex problems with no spurious local
minima are variants of the \emph{low-rank matrix recovery} problem.
The simplest version (known as \emph{matrix sensing}) seeks to recover
an $n\times n$ positive semidefinite matrix $Z$ of low rank $r\ll n$,
given measurement matrices $A_{1},\ldots,A_{m}$ and noiseless data
$b_{i}=\langle A_{i},Z\rangle$. The usual, nonconvex approach is
to solve the following
\begin{equation}
\underset{x\in\R^{n\times r}}{\text{minimize }}\|\AA(xx^{T})-b\|^{2}\quad\text{ where }\quad\AA(X)=\begin{bmatrix}\langle A_{1},X\rangle & \cdots & \langle A_{m},X\rangle\end{bmatrix}^{T}\label{eq:lrmr}
\end{equation}
to second-order optimality, using a local search algorithm like (stochastic)
gradient descent~\cite{ge2015escaping,jin2017escape} and trust region
Newton's method~\cite{cartis2012complexity,boumal2018global}, starting
from a random initial point. 

Exact recovery of the ground truth $Z$ is guaranteed under the assumption
that $\AA$ satisfies the \emph{restricted isometry property}~\cite{candes2005decoding,candes2006stable,recht2010guaranteed,candes2011tight}
with a sufficiently small constant. The original result is due to
Bhojanapalli~et~al.~\cite{bhojanapalli2016global}, though we adapt
the statement below from a later result by Ge~et~al.~\cite[Theorem~8]{ge2017nospurious}.
(Zhu~et~al.~\cite{zhu2018global} give an equivalent statement
for nonsymmetric matrices.) 
\begin{defn}[Restricted Isometry Property]
The linear map $\AA:\R^{n\times n}\to\R^{m}$ is said to satisfy
$(r,\delta_{r})$\nobreakdash-RIP with constant $0\le\delta_{r}<1$
if there exists a fixed scaling $\gamma>0$ such that for all rank-$r$
matrices $X$:
\begin{equation}
(1-\delta_{r})\|X\|_{F}^{2}\le\gamma\cdot\|\AA(X)\|^{2}\le(1+\delta_{r})\|X\|_{F}^{2}.\label{eq:rip}
\end{equation}
We say that $\AA$ satisfies $r$\nobreakdash-RIP if $\AA$ satisfies
$(r,\delta_{r})$\nobreakdash-RIP with some $\delta_{r}<1$.
\end{defn}
\begin{thm}[No spurious local minima]
\label{thm:exact_recovery}Let $\AA$ satisfy $(2r,\delta_{2r})$\nobreakdash-RIP
with $\delta_{2r}<1/5$. Then, (\ref{eq:lrmr}) has no spurious local
minima: every local minimum $x$ satisfies $xx^{T}=Z$, and every
saddle point has an escape (the Hessian has a negative eigenvalue).
Hence, any algorithm that converges to a second-order critical point
is guaranteed to recover $Z$ exactly.
\end{thm}
Standard proofs of Theorem~\ref{thm:exact_recovery} use a \emph{norm-preserving}
argument: if $\AA$ satisfies $(2r,\delta_{2r})$\nobreakdash-RIP
with a small constant $\delta_{2r}$, then we can view the least-squares
residual $\AA(xx^{T})-b$ as a dimension-reduced embedding of the
displacement vector $xx^{T}-Z$, as in
\begin{equation}
\|\AA(xx^{T})-b\|^{2}=\|\AA(xx^{T}-Z)\|^{2}\approx\|xx^{T}-Z\|_{F}^{2}\text{ up to scaling.}\label{eq:rsdp_full}
\end{equation}
The high-dimensional problem of minimizing $\|xx^{T}-Z\|_{F}^{2}$
over $x$ contains no spurious local minima, so its dimension-reduced
embedding (\ref{eq:lrmr}) should satisfy a similar statement. Indeed,
this same argument can be repeated for noisy measurements and nonsymmetric
matrices to result in similar guarantees~\cite{bhojanapalli2016global,ge2017nospurious}.

The norm-preserving argument also extends to ``harder'' choices
of $\AA$ that do not satisfy RIP over its entire domain. In the matrix
completion problem, the RIP-like condition $\|\AA(X)\|^{2}\approx\|X\|_{F}^{2}$
holds only when $X$ is both low-rank and sufficiently dense~\cite{candes2009exact}.
Nevertheless, Ge~et~al.~\cite{ge2016matrix} proved a similar result
to Theorem~\ref{thm:exact_recovery} for this problem, by adding
a regularizing term to the objective. For a detailed introduction
to the norm-preserving argument and its extension with regularizers,
we refer the interested reader to~\cite{ge2016matrix,ge2017nospurious}.

\subsection{How much restricted isometry?}

The RIP threshold $\delta_{2r}<1/5$ in Theorem~\ref{thm:exact_recovery}
is highly conservative\textemdash it is only applicable to nearly-isotropic
measurements like Gaussian measurements. Let us put this point into
perspective by measuring distortion using the \emph{condition number}\footnote{Given a linear map, the condition number measures the ratio in size
between the largest and smallest images, given a unit-sized input.
Within our specific context, the $2r$-restricted condition number
is the smallest $\kappa_{2r}=L/\ell$ such that $\ell\|X\|_{F}^{2}\le\|\AA(X)\|^{2}\le L\|X\|_{F}^{2}$
holds for all rank-$2r$ matrices $X$.} $\kappa_{2r}\in[1,\infty)$. Deterministic linear maps from real-life
applications usually have condition numbers $\kappa_{2r}$ between
$10^{2}$ and $10^{4}$, and these translate to RIP constants $\delta_{2r}=(\kappa_{2r}-1)/(\kappa_{2r}+1)$
between $0.99$ and $0.9999$. By contrast, the RIP threshold $\delta_{2r}<1/5$
requires an equivalent condition number of $\kappa_{2r}=(1+\delta_{2r})/(1-\delta_{2r})<3/2$,
which would be considered \emph{near-perfect} in linear algebra. 

In practice, nonconvex matrix completion works for a much wider class
of problems than those suggested by Theorem~\ref{thm:exact_recovery}~\cite{bottou2008tradeoffs,bottou2010large,recht2013parallel,agarwal2014reliable}.
Indeed, assuming only that $\AA$ satisfies $2r$\nobreakdash-RIP,
solving (\ref{eq:lrmr}) to global optimality is enough to guarantee
exact recovery~\cite[Theorem 3.2]{recht2010guaranteed}. In turn,
stochastic algorithms like stochastic gradient descent (SGD) are often
able to attain global optimality. This disconnect between theory and
practice motivates the following question.

\textbf{Can Theorem~\ref{thm:exact_recovery} be substantially improved\textemdash is
it possible to guarantee the inexistence of spurious local minima
with $(2r,\delta_{2r})$}\nobreakdash-\textbf{RIP and any value of
$\delta_{2r}<1$? }

At a basic level, the question gauges the generality and usefulness
of RIP as a base assumption for nonconvex recovery. Every family of
measure operators $\AA$\textemdash even correlated and ``bad''
measurement ensembles\textemdash will eventually come to satisfy $2r$\nobreakdash-RIP
as the number of measurements $m$ grows large. Indeed, given $m\ge n(n+1)/2$
linearly independent measurements, the operator $\AA$ becomes invertible,
and hence trivially $2r$\nobreakdash-RIP. In this limit, recovering
the ground truth $Z$ from noiseless measurements is as easy as solving
a system of linear equations. Yet, it remains unclear whether nonconvex
recovery is guaranteed to succeed.

At a higher level, the question also gauges the wisdom of exact recovery
guarantees through ``no spurious local minima''. It may be sufficient
but not necessary; exact recovery may actually hinge on SGD's ability
to avoid and escape spurious local minima when they do exist. Indeed,
there is growing empirical evidence that SGD outmaneuvers the ``optimization
landscape'' of nonconvex functions~\cite{bottou2008tradeoffs,bottou2010large,krizhevsky2012imagenet,recht2013parallel,agarwal2014reliable},
and achieves some global properties~\cite{hardt2016train,zhang2016understanding,wilson2017marginal}.
It remains unclear whether the success of SGD for matrix recovery
should be attributed to the inexistence of spurious local minima,
or to some global property of SGD.

\subsection{Our results}

In this paper, we give a strong negative answer to the question above.
Consider the counterexample below, which satisfies $(2r,\delta_{2r})$\nobreakdash-RIP
with $\delta_{2r}=1/2$, but nevertheless contains a spurious local
minimum that causes SGD to fail in 12\% of trials.
\begin{example}
\label{exa:simple}Consider the following $(2,1/2)$\nobreakdash-RIP
instance of (\ref{eq:lrmr}) with matrices
\[
Z=\begin{bmatrix}1 & 0\\
0 & 0
\end{bmatrix},\quad A_{1}=\begin{bmatrix}\sqrt{2} & 0\\
0 & 1/\sqrt{2}
\end{bmatrix},\quad A_{2}=\begin{bmatrix}0 & \sqrt{3/2}\\
\sqrt{3/2} & 0
\end{bmatrix},\quad A_{3}=\begin{bmatrix}0 & 0\\
0 & \sqrt{3/2}
\end{bmatrix}.
\]
Note that the associated operator $\AA$ is invertible and satisfies
$\|X\|_{F}^{2}\le\|\AA(X)\|^{2}\le3\|X\|_{F}^{2}$ for all $X$. Nevertheless,
the point $x=(0,1/\sqrt{2})$ satisfies second-order optimality,
\[
f(x)\equiv\|\AA(xx^{T}-Z)\|^{2}=\frac{3}{2},\qquad\nabla f(x)=\begin{bmatrix}0\\
0
\end{bmatrix},\qquad\nabla^{2}f(x)=\begin{bmatrix}0 & 0\\
0 & 8
\end{bmatrix},
\]
and randomly initialized SGD can indeed become stranded around this
point, as shown in Figure~\ref{fig:simple}. Repeating these trials
100,000 times yields 87,947 successful trials, for a failure rate
of $12.1\pm0.3\%$ to three standard deviations. 
\end{example}
\begin{figure}
\includegraphics[width=0.5\columnwidth]{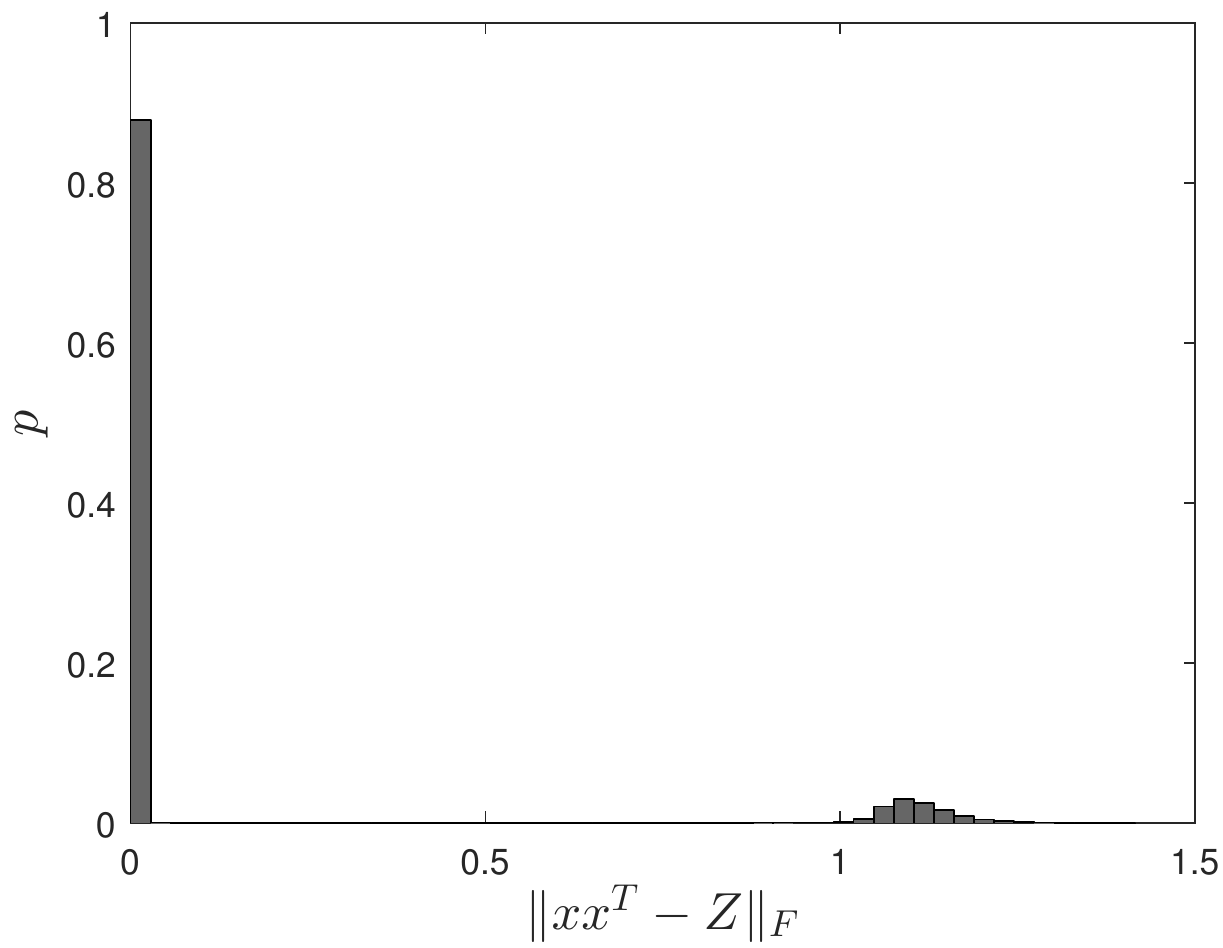}\hspace*{\fill}\includegraphics[width=0.5\columnwidth]{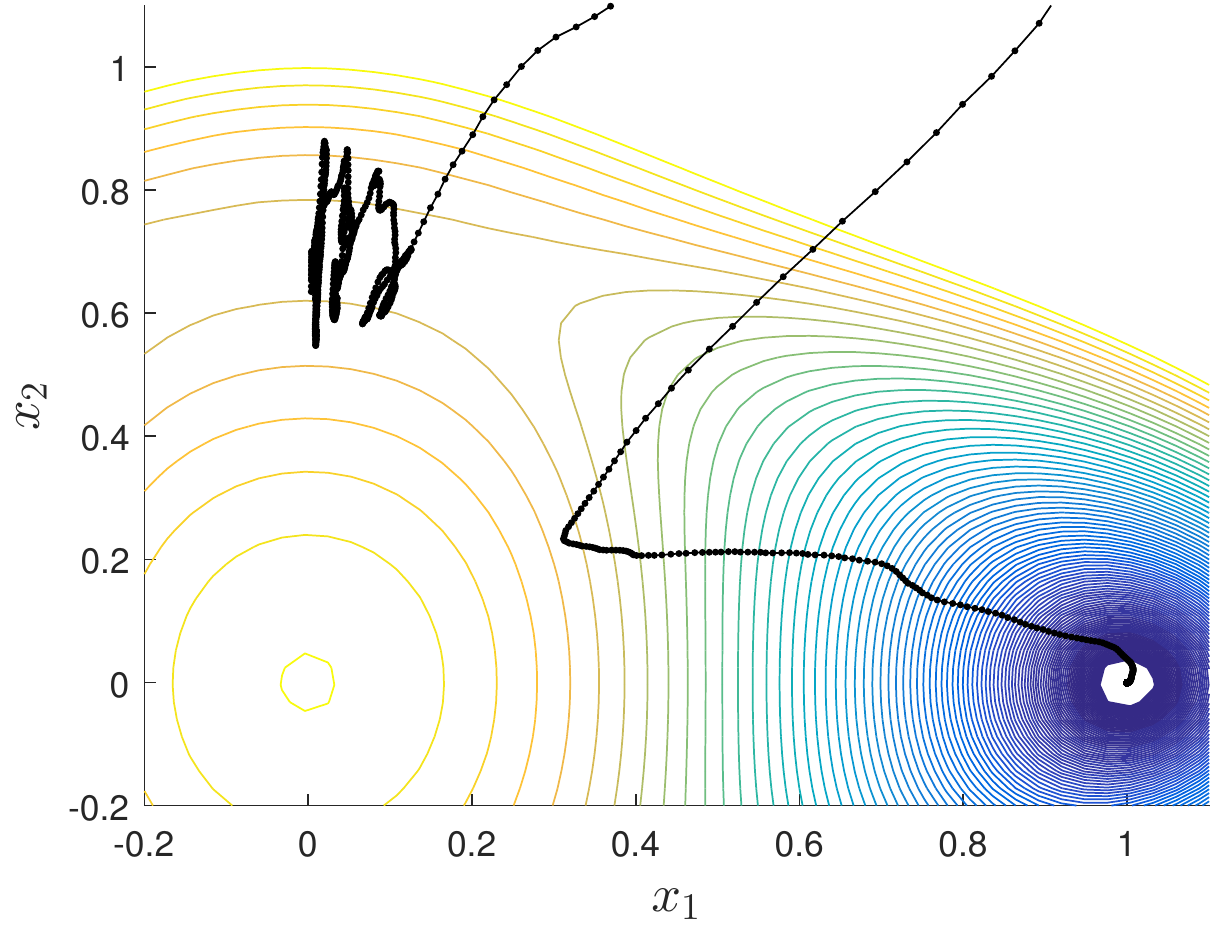}

\caption{\label{fig:simple}Solving Example~\ref{exa:simple} using stochastic
gradient descent randomly initialized with the standard Gaussian.
\textbf{(Left)} Histogram over 100,000 trials of final error $\|xx^{T}-Z\|_{F}$
after $10^{3}$ steps with learning rate $\alpha=10^{-3}$ and momentum
$\beta=0.9$. \textbf{(Right)} Two typical stochastic gradient descent
trajectories, showing convergence to the spurious local minimum at
$(0,1/\sqrt{2})$, and to the ground truth at $(1,0)$.}
\end{figure}
Accordingly, RIP-based exact recovery guarantees like Theorem~\ref{thm:exact_recovery}
cannot be improved beyond $\delta_{2r}<1/2$. Otherwise, spurious
local minima can exist, and SGD may become trapped. Using a local
search algorithm with a random initialization, ``no spurious local
minima'' is not only sufficient for exact recovery, but also necessary.

In fact,\emph{ }there exists an infinite number of counterexamples
like Example~\ref{exa:simple}. In Section~\ref{sec:Solutions},
we prove that, in the rank-1 case, \emph{almost every} choice of $x,Z$
generates an instance of (\ref{eq:lrmr}) with a strict spurious local
minimum.
\begin{thm}[Informal]
\label{thm:informal}Let $x,z\in\R^{n}$ be nonzero and not colinear.
Then, there exists an instance of (\ref{eq:lrmr}) satisfying $(n,\delta_{n})$\nobreakdash-RIP
with $\delta_{n}<1$ that has $Z=zz^{T}$ as the ground truth and
$x$ as a strict spurious local minimum, i.e. with zero gradient and
a positive definite Hessian. Moreover, $\delta_{n}$ is bounded in
terms of the length ratio $\rho=\|x\|/\|z\|$ and the incidence angle
$\phi$ satisfying $x^{T}z=\|x\|\|z\|\cos\phi$ as 
\[
\delta_{n}\le\frac{\tau+\sqrt{1-\zeta^{2}}}{\tau+1}\qquad\text{where }\zeta=\frac{\sin^{2}\phi}{\sqrt{(\rho^{2}-1)^{2}+2\rho^{2}\sin^{2}\phi}},\quad\tau=\frac{2\sqrt{\rho^{2}+\rho^{-2}}}{\zeta^{2}}
\]
\end{thm}
It is therefore impossible to establish ``no spurious local minima''
guarantees unless the RIP constant $\delta$ is small. This is a strong
negative result on the generality and usefulness of RIP as a base
assumption, and also on the wider norm-preserving argument described
earlier in the introduction. In Section~\ref{sec:MinDelta}, we provide
strong empirical evidence for the following \emph{sharp} version of
Theorem~\ref{thm:exact_recovery}. 
\begin{conjecture}
Let $\AA$ satisfy $(2r,\delta_{2r})$\nobreakdash-RIP with $\delta_{2r}<1/2$.
Then, (\ref{eq:lrmr}) has no spurious local minima. Moreover, the
figure of $1/2$ is sharp due to the existence of Example~\ref{exa:simple}.
\end{conjecture}
How is the \emph{practical} performance of SGD affected by spurious
local minima? In Section~\ref{sec:SGDescape}, we apply randomly
initialized SGD to instances of (\ref{eq:lrmr}) engineered to contain
spurious local minima. In one case, SGD recovers the ground truth
with a 100\% success rate, as if the spurious local minima did not
exist. But in another case, SGD fails in 59 of 1,000 trials, for a
positive failure rate of $5.90\pm2.24\%$ to three standard deviations.
Examining the failure cases, we observe that SGD indeed becomes trapped
around a spurious local minimum, similar to Figure~\ref{fig:simple}
in Example~\ref{exa:simple}.

\subsection{Related work}

There have been considerable recent interest in understanding the
empirical ``hardness'' of nonconvex optimization, in view of its
well-established theoretical difficulties. Nonconvex functions contain
saddle points and spurious local minima, and local search algorithms
may become trapped in them. Recent work have generally found the matrix
sensing problem to be ``easy'', particularly under an RIP-like incoherence
assumption. Our results in this paper counters this intuition, showing\textemdash perhaps
surprisingly\textemdash that the problem is generically ``hard''
even under RIP. 

\textbf{Comparison to convex recovery.} Classical theory for the low-rank
matrix recovery problem is based on convex relaxation: replacing $xx^{T}$
in (\ref{eq:lrmr}) by a convex term $X\succeq0$, and augmenting
the objective with a trace penalty $\lambda\cdot\tr(X)$ to induce
a low-rank solution~\cite{candes2009exact,recht2010guaranteed,candes2010power,candes2011tight}.
The convex approach enjoys RIP-based exact recovery guarantees~\cite{candes2011tight},
but these are also fundamentally restricted to small RIP constants~\cite{cai2013sharp,wang2013bounds}\textemdash in
direct analogy with our results for nonconvex recovery. In practice,
convex recovery is usually much more expensive than nonconvex recovery,
because it requires optimizing over an $n\times n$ matrix variable
instead of an $n\times r$ vector-like variable. On the other hand,
it is statistically consistent~\cite{bach2008consistency}, and guaranteed
to succeed with $m\ge\frac{1}{2}n(n+1)$ noiseless, linearly independent
measurements. By comparison, our results show that nonconvex recovery
can still fail in this regime.

\textbf{Convergence to spurious local minima.} Recent results on ``no
spurious local minima'' are often established using a norm-preserving
argument: the problem at hand is the low-dimension embedding of a
canonical problem known to contain no spurious local minima~\cite{ge2015escaping,sun2015complete,sun2016geometric,bhojanapalli2016global,ge2016matrix,ge2017nospurious,park2017non,zhu2018global}.
While the approach is widely applicable in its scope, our results
in this paper finds it to be restrictive in the problem data. More
specifically, the measurement matrices $A_{1},\ldots,A_{m}$ must
come from a nearly-isotropic ensemble like the Gaussian and the sparse
binary.

\textbf{Special initialization schemes. }An alternative way to guarantee
exact recovery is to place the initial point sufficiently close to
the global optimum~\cite{keshavan2010matrix,keshavan2010matrixnoisy,jain2013low,zheng2015convergent,zhao2015nonconvex,sun2016guaranteed}.
This approach is more general because it does not require a global
``no spurious local minima'' guarantee. On the other hand, good
initializations are highly problem-specific and difficult to generalize.
Our results show that spurious local minima can exist arbitrarily
close to the solution. Hence, exact recovery guarantees must give
proof of local attraction, beyond simply starting close to the ground
truth.

\textbf{Ability of SGD to escape spurious local minima. }Practitioners
have long known that stochastic gradient descent (SGD) enjoys properties
inherently suitable for the sort of nonconvex optimization problems
that appear in machine learning~\cite{krizhevsky2012imagenet,bottou2008tradeoffs},
and that it is well-suited for generalizing unseen data~\cite{hardt2016train,zhang2016understanding,wilson2017marginal}.
Its specific behavior is yet not well understood, but it is commonly
conjectured that SGD outperforms classically ``better'' algorithms
like BFGS because it is able to avoid and escape spurious local minima.
Our empirical findings in Section~\ref{sec:SGDescape} partially
confirms this suspicion, showing that randomly initialized SGD is
sometimes able to avoid and escape spurious local minima as if they
did not exist. In other cases, however, SGD can indeed become stuck
at a local minimum, thereby resulting in a positive failure rate. 

\subsection*{Notation}

We use $x$ to refer to any candidate point, and $Z=zz^{T}$ to refer
to a rank-$r$ factorization of the ground truth $Z$. For clarity,
we use lower-case $x,z$ even when these are $n\times r$ matrices. 

The sets $\R^{n\times n}\supset\S^{n}$ are the space of $n\times n$
real matrices and real symmetric matrices, and $\langle X,Y\rangle\equiv\tr(X^{T}Y)$
and $\|X\|_{F}^{2}\equiv\langle X,X\rangle$ are the Frobenius inner
product and norm. We write $X\succeq0$ (resp. $X\succ0$) if $X$
is positive semidefinite (resp. positive definite). Given a matrix
$M$, its spectral norm is $\|M\|$, and its eigenvalues are $\lambda_{1}(M),\ldots,\lambda_{n}(M)$.
If $M=M^{T}$, then $\lambda_{1}(M)\ge\cdots\ge\lambda_{n}(M)$ and
$\lambda_{\max}(M)\equiv\lambda_{1}(M)$, $\lambda_{\min}(M)\equiv\lambda_{n}(M)$.
If $M$ is invertible, then its condition number is $\cond(M)=\|M\|\|M^{-1}\|$;
if not, then $\cond(M)=\infty$. 

The vectorization operator $\vec:\R^{n\times n}\to\R^{n^{2}}$ preserves
inner products $\langle X,Y\rangle=\vec(X)^{T}\vec(Y)$ and Euclidean
norms $\|X\|_{F}=\|\vec(X)\|$. In each case, the matricization operator
$\mat(\cdot)$ is the inverse of $\vec(\cdot)$.

\section{\label{sec:LMI_formu}Key idea: Spurious local minima via convex
optimization}

Given arbitrary $x\in\R^{n\times r}$ and rank-$r$ positive semidefinite
matrix $Z\in\S^{n}$, consider the problem of finding an instance
of (\ref{eq:lrmr}) with $Z$ as the ground truth and $x$ as a spurious
local minimum. While not entirely obvious, this problem is actually
convex, because the first- and second-order optimality conditions
associated with (\ref{eq:lrmr}) are \emph{linear matrix inequality}
(LMI) constraints~\cite{boyd1994linear} with respect to the \emph{kernel}
operator $\HH\equiv\AA^{T}\AA$. The problem of finding an instance
of (\ref{eq:lrmr}) that also satisfies RIP is indeed nonconvex. However,
we can use the \emph{condition number} of $\HH$ as a surrogate for
the RIP constant $\delta$ of $\AA$: if the former is finite, then
the latter is guaranteed to be less than 1. The resulting optimization
is convex, and can be numerically solved using an interior-point method,
like those implemented in SeDuMi~\cite{sturm1999using}, SDPT3~\cite{toh1999sdpt3},
and MOSEK~\cite{andersen2000mosek}, to high accuracy. 

We begin by fixing some definitions. Given a choice of $\AA:\S^{n}\to\R^{m}$
and the ground truth $Z=zz^{T}$, we define the nonconvex objective
\begin{align}
f:\R^{n\times r} & \to\R\qquad\qquad\text{such that} & f(x) & =\|\AA(xx^{T}-zz^{T})\|^{2}\label{eq:fdef}
\end{align}
whose value is always nonnegative by construction. If the point $x$
attains $f(x)=0$, then we call it a \emph{global minimum}; otherwise,
we call it a \emph{spurious} point. Under RIP, $x$ is a global minimum
if and only if $xx^{T}=zz^{T}$~\cite[Theorem 3.2]{recht2010guaranteed}.
The point $x$ is said to be a \emph{local minimum} if $f(x)\le f(x')$
holds for all $x'$ within a local neighborhood of $x$. If $x$ is
a local minimum, then it must satisfy the first and second-order \emph{necessary}
optimality conditions (with some fixed $\mu\ge0$):
\begin{align}
\langle\nabla f(x),u\rangle & =2\langle\AA(xx^{T}-zz^{T}),\AA(xu^{T}+ux^{T})\rangle=0 & \forall u & \in\R^{n\times r},\label{eq:foc}\\
\langle\nabla^{2}f(x)u,u\rangle & =2\langle\AA(xx^{T}-zz^{T}),uu^{T}\rangle+\|\AA(xu^{T}+ux^{T})\|^{2}\ge\mu\|u\|_{F}^{2} & \forall u & \in\R^{n\times r}.\label{eq:soc}
\end{align}
Conversely, if $x$ satisfies the second-order \emph{sufficient} optimality
conditions, that is (\ref{eq:foc})-(\ref{eq:soc}) with $\mu>0$,
then it is a local minimum. Local search algorithms are only guaranteed
to converge to a \emph{first-order critical point} $x$ satisfying
(\ref{eq:foc}), or a \emph{second-order critical point} $x$ satisfying
(\ref{eq:foc})-(\ref{eq:soc}) with $\mu\ge0$. The latter class
of algorithms include stochastic gradient descent~\cite{ge2015escaping},
randomized and noisy gradient descent~\cite{ge2015escaping,lee2016gradient,jin2017escape,du2017gradient},
and various trust-region methods~\cite{conn2000trust,nesterov2006cubic,cartis2012complexity,boumal2018global}.

Given arbitrary choices of $x,z\in\R^{n\times r}$, we formulate the
problem of picking an $\AA$ satisfying (\ref{eq:foc}) and (\ref{eq:soc})
as an LMI feasibility. First, we define $\A=[\vec(A_{1}),\ldots,\vec(A_{m})]^{T}$
satisfying $\A\cdot\vec(X)=\AA(X)$ for all $X$ as the matrix representation
of the operator $\AA$. Then, we rewrite (\ref{eq:foc}) and (\ref{eq:soc})
as $2\cdot\L(\A^{T}\A)=0$ and $2\cdot\M(\A^{T}\A)\succeq\mu I$,
where the linear operators $\L$ and $\M$ are defined
\begin{align}
 & \L:\S^{n^{2}}\to\R^{n\times r}\qquad\text{such that} & \L(\H) & \equiv2\cdot\X^{T}\H e,\label{eq:Ldef}\\
 & \M:\S^{n^{2}}\to\S^{nr\times nr}\qquad\text{such that} & \M(\H) & \equiv2\cdot[I_{r}\otimes\mat(\H e)^{T}]+\X^{T}\H\X,\label{eq:Mdef}
\end{align}
with respect to the error vector $e=\vec(xx^{T}-zz^{T})$ and the
$n^{2}\times nr$ matrix $\X$ that implements the symmetric product
operator $\X\cdot\vec(u)=\vec(xu^{T}+ux^{T})$. To compute a choice
of $\A$ satisfying $\L(\A^{T}\A)=0$ and $\M(\A^{T}\A)\succeq0$,
we solve the following LMI feasibility problem
\begin{equation}
\underset{\H}{\text{maximize }}\quad0\qquad\text{ subject to }\qquad\L(\H)=0,\quad\M(\H)\succeq\mu I,\quad\H\succeq0,\label{eq:lmifeas}
\end{equation}
and factor a feasible $\H$ back into $\A^{T}\A$, e.g. using Cholesky
factorization or an eigendecomposition. Once a matrix representation
$\A$ is found, we recover the matrices $A_{1},\ldots,A_{m}$ implementing
the operator $\AA$ by matricizing each row of $\A$.

Now, the problem of picking $\AA$ with the smallest condition number
may be formulated as the following LMI optimization
\begin{equation}
\underset{\H,\eta}{\text{maximize }}\quad\eta\qquad\text{ subject to }\qquad\eta I\preceq\H\preceq I,\quad\L(\H)=0,\quad\M(\H)\succeq\mu I,\quad\H\succeq0,\label{eq:lmiopt}
\end{equation}
with solution $\H^{\star},\eta^{\star}$. Then, $1/\eta^{\star}$
is the best condition number achievable, and any $\AA$ recovered
from $\H^{\star}$ will satisfy
\[
\left(1-\frac{1-\eta^{\star}}{1+\eta^{\star}}\right)\|X\|^{2}\le\frac{2}{1+\eta^{\star}}\|\AA(X)\|_{F}^{2}\le\left(1+\frac{1-\eta^{\star}}{1+\eta^{\star}}\right)\|X\|^{2}
\]
for all $X$, that is, with \emph{any rank}. As such, $\AA$ is $(n,\delta_{n})$\nobreakdash-RIP
with $\delta_{n}=(1-\eta^{\star})/(1+\eta^{\star})$, and hence also
$(p,\delta_{p})$\nobreakdash-RIP with $\delta_{p}\le\delta_{n}$
for all $p\in\{1,\ldots,n\}$; see e.g.~\cite{recht2010guaranteed,candes2011tight}.
If the optimal value $\eta^{\star}$ is strictly positive, then the
recovered $\AA$ yields an RIP instance of (\ref{eq:lrmr}) with $zz^{T}$
as the ground truth and $x$ as a spurious local minimum, as desired.

It is worth emphasizing that a small condition number\textemdash a
large $\eta^{\star}$ in (\ref{eq:lmiopt})\textemdash will always
yield a small RIP constant $\delta_{n}$, which then bounds all other
RIP constants via $\delta_{n}\ge\delta_{p}$ for all $p\in\{1,\ldots,n\}$.
However, the converse direction is far less useful, as the value of
$\delta_{n}=1$ does not preclude $\delta_{p}$ with $p<n$ from being
small. 

\section{\label{sec:Solutions}Closed-form solutions}

It turns out that the LMI problem (\ref{eq:lmiopt}) in the rank-1
case is sufficiently simple that it can be solved in closed-form.
(All proofs are given in the Appendix.) Let $x,z\in\R^{n}$ be arbitrary
nonzero vectors, and define
\begin{align}
\rho & \equiv\frac{\|x\|}{\|z\|}, & \phi & \equiv\arccos\left(\frac{x^{T}z}{\|x\|\|z\|}\right),\label{eq:rho_phi_def}
\end{align}
as their associated length ratio and incidence angle. We begin by
examining the prevalence of spurious critical points.
\begin{thm}[First-order optimality]
\label{thm:foc}The best-conditioned $\H^{\star}\succeq0$ such that
$\L(\H^{\star})=0$ satisfies
\begin{equation}
\cond(\H^{\star})=\frac{1+\sqrt{1-\zeta^{2}}}{1-\sqrt{1-\zeta^{2}}}\qquad\text{ where }\qquad\zeta\equiv\frac{\sin\phi}{\sqrt{(\rho^{2}-1)^{2}+2\rho^{2}\sin^{2}\phi}}.\label{eq:lb}
\end{equation}
Hence, if $\phi\ne0$, then $x$ is a first-order critical point for
an instance of (\ref{eq:lrmr}) satisfying $(2,\delta)$\nobreakdash-RIP
with $\delta=\sqrt{1-\zeta^{2}}<1$ given in (\ref{eq:lb}).
\end{thm}
The point $x=0$ is always a local maximum for $f$, and hence a spurious
first-order critical point. With a perfect RIP constant $\delta=0$,
Theorem~\ref{thm:foc} says that $x=0$ is also the only spurious
first-order critical point. Otherwise, spurious first-order critical
points may exist elsewhere, even when the RIP constant $\delta$ is
arbitrarily close to zero. This result highlights the importance of
converging to second-order optimality, in order to avoid getting stuck
at a spurious first-order critical point. 

Next, we examine the prevalence of spurious local minima. 
\begin{thm}[Second-order optimality]
\label{thm:soc}There exists $\H$ satisfying $\L(\H)=0,$ $\M(\H)\succeq\mu I,$
and $\eta I\preceq\H\preceq I$ where
\[
\eta\ge\frac{1}{1+\tau}\cdot\left(\frac{1+\sqrt{1-\zeta^{2}}}{1-\sqrt{1-\zeta^{2}}}\right),\qquad\mu=\frac{\|z\|^{2}}{1+\tau},\qquad\tau\equiv\frac{2\sqrt{\rho^{2}+\rho^{-2}}}{\zeta^{2}}
\]
and $\zeta$ is defined in (\ref{eq:lb}). Hence, if $\phi\ne0$ and
$\rho>0$ is finite, then $x$ is a strict local minimum for an instance
of (\ref{eq:lrmr}) satisfying $(2,\delta)$\nobreakdash-RIP with
$\delta=(\tau+\sqrt{1-\zeta^{2}})/(1+\tau)<1$.
\end{thm}
If $\phi\ne0$ and $\rho>0$, then $x$ is guaranteed to be a strict
local minimum for a problem instance satisfying $2$\nobreakdash-RIP.
Hence, we must conclude that spurious local minima are ubiquitous.
The associated RIP constant $\delta<1$ is not too much worse than
than the figure quoted in Theorem~\ref{thm:foc}. On the other hand,
spurious local minima must cease to exist once $\delta<1/5$ according
to Theorem~\ref{thm:exact_recovery}. 

\section{\label{sec:MinDelta}Experiment 1: Minimum $\delta$ with spurious
local minima}

What is smallest RIP constant $\delta_{2r}$ that still admits an
instance of (\ref{eq:lrmr}) with spurious local minima? Let us define
the threshold value as the following 
\begin{gather}
\delta^{\star}=\min_{x,Z,\AA}\{\delta:\nabla f(x)=0,\quad\nabla^{2}f(x)\succeq0,\quad\AA\text{ satisfies }(2r,\delta)\text{-RIP}\}.\label{eq:delta_star}
\end{gather}
Here, we write $f(x)=\|\AA(xx^{T}-Z)\|^{2}$, and optimize over the
spurious local minimum $x\in\R^{n\times r},$ the rank-$r$ ground
truth $Z\succeq0$, and the linear operator $\AA:\R^{n\times n}\to\R^{m}$.
Note that $\delta^{\star}$ gives a ``no spurious local minima''
guarantee, due to the inexistence of counterexamples.
\begin{prop}
Let $\AA$ satisfy $(2r,\delta_{2r})$\nobreakdash-RIP. If $\delta_{2r}<\delta^{\star}$,
then (\ref{eq:lrmr}) has no spurious local minimum.
\end{prop}
\begin{proof}
Suppose that (\ref{eq:lrmr}) contained a spurious local minimum $x$
for ground truth $Z$. Then, substituting this choice of $x,Z,\AA$
into (\ref{eq:delta_star}) would contradict the definition of $\delta^{\star}$
as the minimum. 
\end{proof}
Our convex formulation in Section~\ref{sec:LMI_formu} bounds $\delta^{\star}$
from above. Specifically, our LMI problem (\ref{eq:lmiopt}) with
optimal value $\eta^{\star}$ is equivalent to the following variant
of (\ref{eq:delta_star}) 
\begin{gather}
\delta_{\ub}(x,Z)=\min_{\AA}\{\delta:\nabla f(x)=0,\quad\nabla^{2}f(x)\succeq0,\quad\AA\text{ satisfies }(n,\delta)\text{-RIP}\},\label{eq:delta_star-1}
\end{gather}
with optimal value $\delta_{\ub}(x,Z)=(1-\eta^{\star})/(1+\eta^{\star})$.
Now, (\ref{eq:delta_star-1}) gives an upper-bound on (\ref{eq:delta_star})
because $(n,\delta)$\nobreakdash-RIP is a \emph{sufficient} condition
for $(2r,\delta)$\nobreakdash-RIP. Hence, we have $\delta_{\ub}(x,Z)\ge\delta^{\star}$
for every valid choice of $x$ and $Z$.

The same convex formulation can be modified to bound $\delta^{\star}$
from below\footnote{We thank an anonymous reviewer for this key insight.}.
Specifically, a necessary condition for $\AA$ to satisfy $(2r,\delta_{2r})$\nobreakdash-RIP
is the following
\begin{equation}
(1-\delta_{2r})\|UYU^{T}\|_{F}^{2}\le\|\AA(UYU^{T})\|^{2}\le(1+\delta_{2r})\|UYU^{T}\|_{F}^{2}\qquad\forall Y\in\R^{2r\times2r}\label{eq:RIP-nec}
\end{equation}
where $U$ is a \emph{fixed} $n\times2r$ matrix. This is a convex
linear matrix inequality; substituting (\ref{eq:RIP-nec}) into (\ref{eq:delta_star})
in lieu of of $(2r,\delta)$\nobreakdash-RIP yields a convex optimization
problem
\begin{gather}
\delta_{\lb}(x,Z,\mathcal{U})=\min_{\AA}\{\delta:\nabla f(x)=0,\quad\nabla^{2}f(x)\succeq0,\quad\text{(\ref{eq:RIP-nec})}\},\label{eq:delta_star-1-1}
\end{gather}
that generates lower-bounds $\delta^{\star}\ge\delta_{\lb}(x,Z,U)$.

\textbf{Our best upper-bound is likely $\delta^{\star}\le1/2$.} The
existence of Example~\ref{exa:simple} gives the upper-bound of $\delta^{\star}\le1/2$.
To improve upon this bound, we randomly sample $x,z\in\R^{n\times r}$
i.i.d. from the standard Gaussian, and evaluate $\delta_{\ub}(x,zz^{T})$
using MOSEK~\cite{andersen2000mosek}. We perform the experiment
for 3 hours on each tuple $(n,r)\in\{1,2,\ldots,10\}\times\{1,2\}$
but obtain $\delta_{\ub}(x,zz^{T})\ge1/2$ for every $x$ and $z$
considered.

\textbf{The threshold is likely $\delta^{\star}=1/2$.} Now, we randomly
sample $x,z\in\R^{n\times r}$ i.i.d. from the standard Gaussian.
For each fixed $\{x,z\}$, we set $U=[x,z]$ and evaluate $\delta_{\lb}(x,Z,U)$
using MOSEK~\cite{andersen2000mosek}. We perform the same experiment
as the above, but find that $\delta_{\lb}(x,zz^{T},U)\ge1/2$ for
every $x$ and $z$ considered. Combined with the existence of the
upper-bound $\delta^{\star}=1/2$, these experiments strongly suggest
that $\delta^{\star}=1/2$. 

\section{\label{sec:SGDescape}Experiment 2: SGD escapes spurious local minima}

\begin{figure}
\begin{minipage}[b]{0.5\linewidth}%
\includegraphics[width=1\columnwidth]{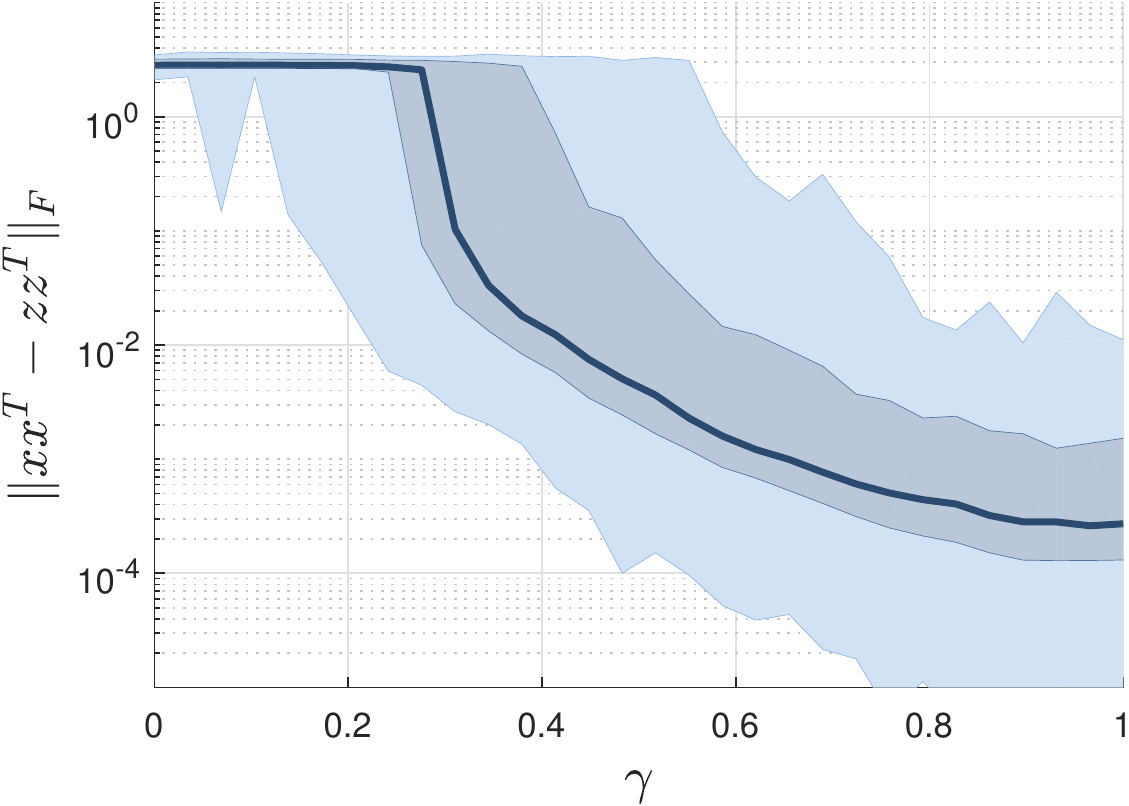}%
\end{minipage}%
\begin{minipage}[b]{0.5\linewidth}%
\includegraphics[width=1\columnwidth]{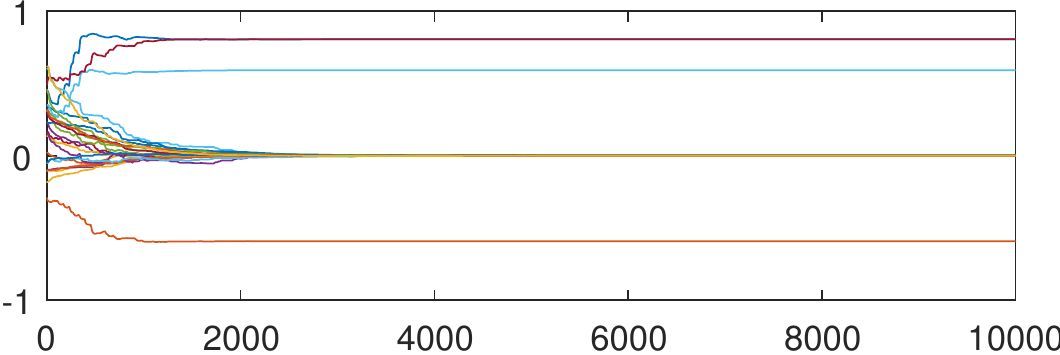}

\includegraphics[width=1\columnwidth]{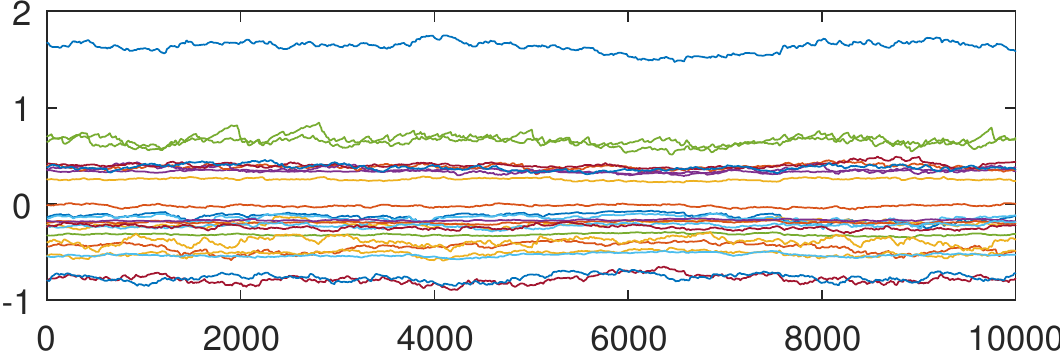}%
\end{minipage}

\caption{\label{fig:bad} ``Bad'' instance ($n=12$, $r=2$) with RIP constant
$\delta=0.973$ and spurious local min at $x_{loc}$ satisfying $\|xx^{T}\|_{F}/\|zz^{T}\|_{F}\approx4$.
Here, $\gamma$ controls initial SGD $x=\gamma w+(1-\gamma)x_{loc}$
where $w$ is random Gaussian. \textbf{(Left)} Error distribution
after 10,000 SGD steps (rate $10^{-4}$, momentum $0.9$) over 1,000
trials. Line: median. Inner bands: 5\%-95\% quantile. Outer bands:
min/max. \textbf{(Right top)} Random initialization with $\gamma=1$;
\textbf{(Right bottom)} Initialization at local min with $\gamma=0$.}

\medskip{}

\begin{minipage}[b]{0.5\linewidth}%
\includegraphics[width=1\columnwidth]{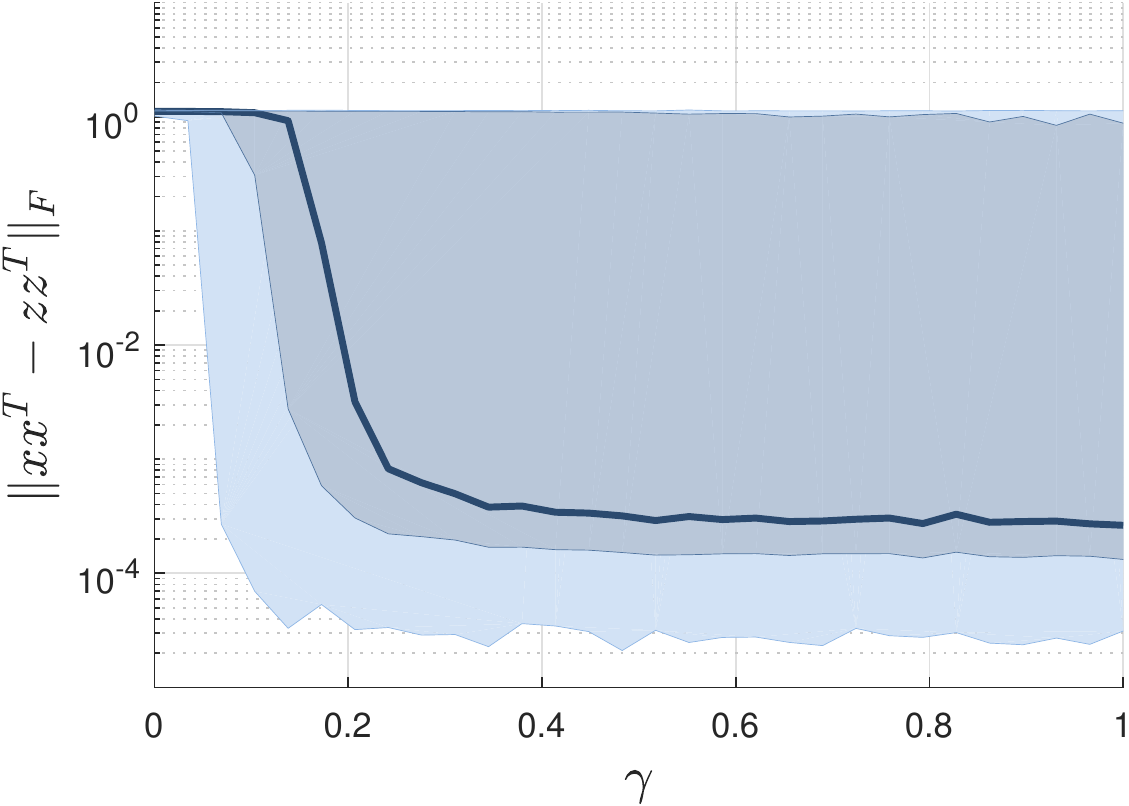}%
\end{minipage}%
\begin{minipage}[b]{0.5\linewidth}%
\includegraphics[width=1\columnwidth]{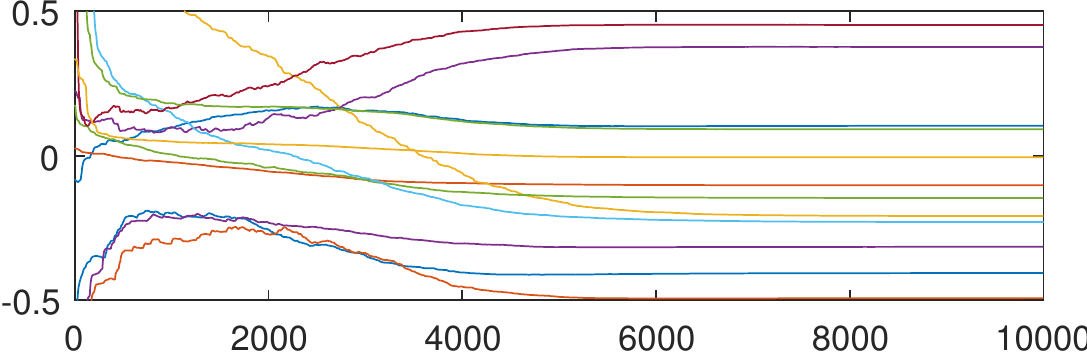}

\includegraphics[width=1\columnwidth]{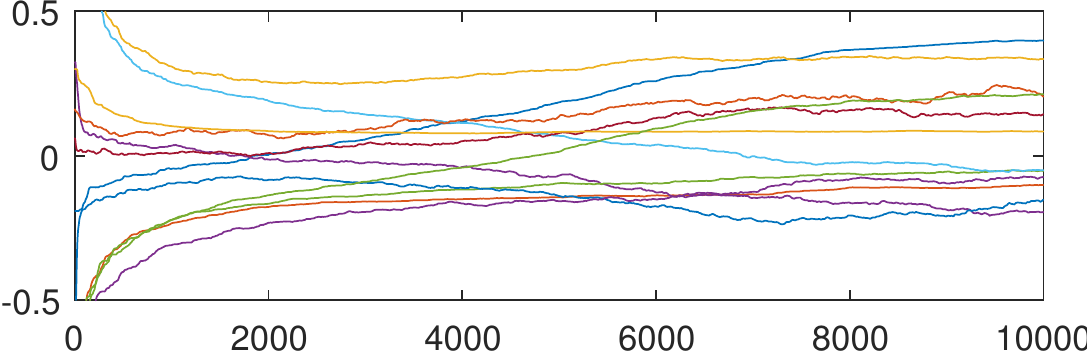}%
\end{minipage}

\caption{\label{fig:good}``Good'' instance ($n=12$, $r=1$) with RIP constant
$\delta=1/2$ and spurious local min at $x_{loc}$ satisfying $\|xx^{T}\|_{F}/\|zz^{T}\|_{F}=1/2$
and $x^{T}z=0$. Here, $\gamma$ controls initial SGD $x=\gamma w+(1-\gamma)x_{loc}$
where $w$ is random Gaussian. \textbf{(Left)} Error distribution
after 10,000 SGD steps (rate $10^{-3}$, momentum $0.9$) over 1,000
trials. Line: median. Inner bands: 5\%-95\% quantile. Outer bands:
min/max. \textbf{(Right top)} Random initialization $\gamma=1$ with
success; \textbf{(Right bottom)} Random initialization $\gamma=1$
with failure.}
\end{figure}
How is the performance of SGD affected by the presence of spurious
local minima? Given that spurious local minima cease to exist with
$\delta<1/5$, we might conjecture that the performance of SGD is
a decreasing function of $\delta$. Indeed, this conjecture is generally
supported by evidence from the nearly-isotropic measurement ensembles~\cite{bottou2008tradeoffs,bottou2010large,recht2013parallel,agarwal2014reliable},
all of which show improving performance with increasing number of
measurements $m$.

This section empirically measures SGD (with momentum, fixed learning
rates, and batchsizes of one) on two instances of (\ref{eq:lrmr})
with different values of $\delta$, both engineered to contain spurious
local minima by numerically solving (\ref{eq:lmiopt}). We consider
a ``bad'' instance, with $\delta=0.975$ and rank $r=2$, and a
``good'' instance, with $\delta=1/2$ and rank $r=1$. The condition
number of the ``bad'' instance is 25 times higher than the ``good''
instance, so classical theory suggests the former to be a factor of
5-25 times harder to solve than the former. Moreover, the ``good''
instance is locally strongly convex at its isolated global minima
while the ``bad'' instance is only locally weakly convex, so first-order
methods like SGD should locally converge at a linear rate for the
former, and sublinearly for the latter.

\textbf{SGD consistently succeeds on ``bad'' instance with $\delta=0.975$
and $r=2$.} We generate the ``bad'' instance by fixing $n=12$,
$r=2$, selecting $x,z\in\R^{n\times r}$ i.i.d. from the standard
Gaussian, rescale $z$ so that $\|zz^{T}\|_{F}=1$ and rescale $x$
so that $\|xx^{T}\|_{F}/\|zz^{T}\|_{F}\approx4$, and solving (\ref{eq:lmiopt});
the results are shown in Figure~\ref{fig:bad}. The results at $\gamma\approx0$
validate $x_{loc}$ as a true local minimum: if initialized here,
then SGD remains stuck here with $>100\%$ error. The results at $\gamma\approx1$
shows randomly initialized SGD either escaping our engineered spurious
local minimum, or avoiding it altogether. All 1,000 trials at $\gamma=1$
recover the ground truth to $<1\%$ accuracy, with 95\% quantile at
$\approx0.6\%$. 

\textbf{SGD consistently fails on ``good'' instance with $\delta=1/2$
and $r=1$.} We generate the ``good'' instance with $n=12$ and
$r=1$ using the procedure in the previous Section; the results are
shown in Figure~\ref{fig:good}. As expected, the results at $\gamma\approx0$
validate $x_{loc}$ as a true local minimum. However, even with $\gamma=1$
yielding a random initialization, 59 of the 1,000 trials still result
in an error of $>50\%$, thereby yielding a failure rate of $5.90\pm2.24\%$
up to three standard deviations. Examine the failed trials closer,
we do indeed find SGD hovering around our engineered spurious local
minimum. 

Repeating the experiment over other instances of (\ref{eq:lrmr})
obtained by solving (\ref{eq:lmiopt}) with randomly selected $x,z$,
we generally obtain graphs that look like Figure~\ref{fig:bad}.
In other words, SGD usually escapes spurious local minima even when
they are engineered to exist. These observations continue to hold
true with even massive condition numbers on the order of $10^{4}$,
with corresponding RIP constant $\delta=1-10^{-4}$. On the other
hand, we do occasionally sample well-conditioned instances that behave
closer to the ``good'' instance describe above, causing SGD to consistently
fail.

\section{Conclusions}

The nonconvex formulation of low-rank matrix recovery is highly effective,
despite the apparent risk of getting stuck at a spurious local minimum.
Recent results have shown that if the linear measurements of the low-rank
matrix satisfy a restricted isometry property (RIP), then the problem
contains\emph{ no spurious local minima}, so exact recovery is guaranteed.
Most of these existing results are based on a norm-preserving argument:
relating $\|\AA(xx^{T}-Z)\|\approx\|xx^{T}-Z\|_{F}$ and arguing that
a lack of spurious local minima in the latter implies a similar statement
in the former. 

Our key message in this paper is that moderate RIP is not enough to
eliminate spurious local minima. To prove this, we formulate a convex
optimization problem in Section~\ref{sec:LMI_formu} that generates
counterexamples that satisfy RIP but contain spurious local minima.
Solving this convex formulation in closed-form in Section~\ref{sec:Solutions}
shows that counterexamples are ubiquitous: almost any rank-1 $Z\succeq0$
and any $x\in\R^{n}$ can respectively be the ground truth and spurious
local minimum to an instance of matrix recovery satisfying RIP. We
gave one specific counterexample with RIP constant $\delta=1/2$ in
the introduction that causes randomly initialized stochastic gradient
descent (SGD) to fail 12\% of the time. 

Moreover, stochastic gradient descent (SGD) is often but not always
able to avoid and escape spurious local minima. In Section~\ref{sec:SGDescape},
randomly initialized SGD solved one example with a 100\% success rate
over 1,000 trials, despite the presence of spurious local minima.
However, it failed with a consistent rate of $\approx6\%$ on another
other example with an RIP constant of just $1/2$. Hence, as long
as spurious local minima exist, we cannot expect to guarantee exact
recovery with SGD (without a much deeper understanding of the algorithm). 

Overall, exact recovery guarantees will generally require a proof
of no spurious local minima. However, arguments based solely on norm
preservation are conservative, because most measurements are not isotropic
enough to eliminate spurious local minima. 

\section*{Acknowledgements}

We thank our three NIPS reviewers for helpful comments and suggestions.
In particular, we thank reviewer \#2 for a key insight that allowed
us to lower-bound $\delta^{\star}$ in Section~\ref{sec:MinDelta}.
This work was supported by the ONR Awards N00014-17-1-2933 and ONR
N00014-18-1-2526, NSF Award 1808859, DARPA Award D16AP00002, and AFOSR
Award FA9550- 17-1-0163.

\bibliographystyle{plain}
\bibliography{euclid}

\appendix

\section{\label{sec:Proof-of-Main}Proofs of Main Results}

Recall that we have defined
\begin{align*}
\L & :\S^{n^{2}}\to\R^{n\times r}\qquad\L(\H)=2\cdot\X^{T}\H e,\\
\L^{T} & :\R^{n\times r}\to\S^{n^{2}}\qquad\L^{T}(y)=ey^{T}\X^{T}+\X ye^{T}
\end{align*}
and also
\begin{align*}
\M & :\S^{n^{2}}\to\S^{nr}\qquad\M(\H)=2\cdot[I_{r}\otimes\mat(\H e)]+\X^{T}\H\X,\\
\M^{T} & :\S^{nr}\to\S^{n^{2}}\qquad\M^{T}(U)=\vec(U)e^{T}+e\vec(U)^{T}+\X U\X^{T}.
\end{align*}
Moreover, we use $\rho=\|x\|/\|z\|$ and $\phi=\arccos(x^{T}z/\|x\|\|z\|)$.

\subsection{Technical lemmas}

We begin by solving an eigenvalue LMI in closed-form.
\begin{lem}
\label{lem:dual_prob}Given $M\in\S^{n}$ with $\tr(M)\ge0$, we split
the matrix into a positive part $M_{+}$ and a negative part $M_{-}$
satisfying
\[
M=M_{+}-M_{-}\qquad\text{where }\quad M_{+},M_{-}\succeq0,\quad M_{+}M_{-}=0.
\]
Then the following problem has solution
\[
\tr(M_{-})/\tr(M_{+})=\min_{\begin{subarray}{c}
\alpha\in\R\\
U,V\succeq0
\end{subarray}}\{\tr(V):\tr(U)=1,\alpha M=U-V\}
\]
\end{lem}
\begin{proof}
Write $p^{\star}$ as the optimal value. Then, 
\begin{align*}
p^{\star}= & \max_{\beta}\min_{\begin{subarray}{c}
\alpha\in\R\\
U,V\succeq0
\end{subarray}}\{\tr(V)+\beta\cdot[\tr(U)-1]:\alpha M=U-V\}\\
= & \max_{\beta\ge0}\min_{\alpha\in\R}\{-\beta+\min_{U,V\succeq0}\{\tr(V)+\beta\cdot\tr(U):\alpha M=U-V\}\}\\
= & \max_{\beta\ge0}\min_{\alpha\in\R}\{-\beta+\alpha\cdot[\tr(M_{-})+\beta\cdot\tr(M_{+})]\}\\
= & \max_{\beta\ge0}\{-\beta:\tr(M_{-})+\beta\cdot\tr(M_{+})=0\}\\
= & \tr(M_{-})/\tr(M_{+}).
\end{align*}
The first line converts an equality constraint into a Lagrangian.
The second line isolates the optimization over $U,V\succeq0$ with
$\beta\ge0$, noting that $\beta<0$ would yield $\tr(U)\to\infty$.
The third line solves the minimization over $U,V\succeq0$ in closed-form.
The fourth line views $\alpha$ as a Lagrange multiplier. 
\end{proof}
The matrix $\L^{T}(y)$ is rank-2 with the following eigenvalues.
\begin{lem}
\label{lem:eig_bound}The matrix $\L^{T}(y)$ is rank-2, and its two
nonzero eigenvalues are
\begin{equation}
\|\X y\|\|e\|(\cos\theta_{y}\pm1),\quad\text{where }\cos\theta_{y}=\frac{e^{T}\X y}{\|e\|\|\X y\|}.\label{eq:eig_bound}
\end{equation}
\end{lem}
\begin{proof}
We project $\X y$ onto $e$ and define $q$ as the residual, as in
$\X y=\alpha e+q$ with $\alpha=(e^{T}\X y)/\|e\|^{2}$. Then we have
the similarity relation
\[
\L^{T}(y)=\begin{bmatrix}e & q\end{bmatrix}\begin{bmatrix}2\alpha & 1\\
1 & 0
\end{bmatrix}\begin{bmatrix}e & q\end{bmatrix}^{T}\sim\|e\|\cdot\begin{bmatrix}2\alpha\|e\| & \|q\|\\
\|q\| & 0
\end{bmatrix},
\]
and the $2\times2$ matrix has eigenvalues $\|\alpha e\|^{2}\pm\sqrt{\|\alpha e\|^{2}+\|q\|^{2}}$.
Substituting $\|\X y\|^{2}=\|\alpha e\|^{2}+\|q\|^{2}$ completes
the proof.
\end{proof}
Also, the angle between $e$ and $\range(\X)$ is closely associated
with the angle between $x$ and $z$. 
\begin{lem}
\label{lem:angle_bound}Define the incidence angle $\theta$ between
$e$ and $\range(\X)$ as
\begin{equation}
\theta=\arccos\left(\max_{y}\frac{e^{T}\X y}{\|e\|\|\X y\|}\right).\label{eq:theta_def}
\end{equation}
Then, the angle has value
\[
\sin\theta=\frac{(\|z\|\sin\phi)^{2}}{\|e\|}=\frac{\sin^{2}\phi}{\sqrt{(\rho^{2}-1)^{2}+2\rho^{2}\sin^{2}\phi}}.
\]
\end{lem}
\begin{proof}
We project $z$ onto $\range(x)$ and define $w$ as the residual,
as in $z=x\alpha+w$ where $\alpha=(x^{T}z)/\|x\|^{2}$. Then, we
have the similarity relation
\[
xx^{T}-zz^{T}=\begin{bmatrix}x & w\end{bmatrix}\begin{bmatrix}(1-\alpha^{2})I_{r} & -\alpha I_{r}\\
-\alpha I_{r} & -I_{r}
\end{bmatrix}\begin{bmatrix}x & w\end{bmatrix}^{T}\sim\begin{bmatrix}(1-\alpha^{2})\|x\|^{2} & -\alpha\|x\|\|w\|\\
-\alpha\|x\|\|w\| & -\|w\|^{2}
\end{bmatrix},
\]
and may solve the problem of projecting $e$ onto $\range(\X)$ after
a change of basis
\begin{align*}
\|e\|\sin\theta= & \min_{y}\|\X y-e\|\\
= & \min_{y}\|xy^{T}+yx^{T}-(xx^{T}-zz^{T})\|_{F},\\
= & \min_{\tilde{y}_{1},\tilde{y}_{2}}\left\Vert \begin{bmatrix}\tilde{y}_{1} & \tilde{y}_{2}\\
\tilde{y}_{2} & 0
\end{bmatrix}-\begin{bmatrix}(1-\alpha^{2})\|x\|^{2} & -\alpha\|x\|\|w\|\\
-\alpha\|x\|\|w\| & -\|w\|^{2}
\end{bmatrix}\right\Vert _{F},\\
= & \|w\|^{2}=\|z\|^{2}\sin^{2}\phi.
\end{align*}
This proves the first equality. On the other hand, we have
\begin{align}
\|e\|=\|xx^{T}-zz\|_{F} & =\sqrt{\|x\|^{4}+\|z\|^{4}-2(x^{T}z)^{2}}=\|z\|^{2}\sqrt{\rho^{4}+1-2\rho^{2}\cos\phi}.\label{eq:e_phirho}
\end{align}
Completing the square and substituting yields the second equality.
\end{proof}
\begin{lem}
\label{lem:soc}Let $\hat{\H}$ be the optimal choice in Theorem~\ref{thm:foc}.
Then
\[
\|\mat(\hat{\H}e)\|\le\sqrt{1+\rho^{4}}\|z\|^{2},\qquad\lambda_{\min}(\X^{T}P_{e\perp}\X)\ge2\|x\|^{2}\zeta^{2}
\]
where $\zeta$ was defined in (\ref{eq:lb}).
\end{lem}
\begin{proof}
For the first bound, we have
\begin{align*}
u^{T}\mat(\hat{\H}e)u=(u\otimes u)^{T}\hat{\H}e & \le\|u\otimes u\|\|\hat{H}\|\|e\|=\|u\|^{2}\|e\|,
\end{align*}
and $\|e\|^{2}=\|z\|^{4}(1-\rho^{2}\cos\phi+\rho^{4})\le\|z\|^{4}(1+\rho^{4})$
from (\ref{eq:e_phirho}). For the second bound, define $\theta$
as the angle between $e$ and $\range(\X)$ in (\ref{lem:angle_bound}),
and note that $\zeta$ in (\ref{eq:lb}) satisfies $\zeta=\sin\theta$
by construction via Lemma~\ref{lem:angle_bound}. Then,
\begin{align*}
v^{T}(\X^{T}P_{e\perp}\X)v & =\|P_{e\perp}\X v\|^{2}\text{ because projections are idempotent: }P_{e\perp}=P_{e\perp}^{2}\\
 & =\min_{\alpha\in\R}\|\X v-e\alpha\|^{2}=\min_{\alpha\in\R}\{\|\X v\|^{2}-2\alpha e^{T}\X v+\alpha^{2}\|e\|^{2}\}\\
 & \ge\min_{\alpha\in\R}\{\|\X v\|^{2}-2\alpha\|e\|\|\X v\|\cos\theta+\alpha^{2}\|e\|^{2}\}\\
 & \qquad\text{whose minimum is attained at }\alpha=\|\X v\|\cos\theta\\
 & =\|\X v\|^{2}(1-\cos^{2}\theta)=\|\X v\|^{2}\sin^{2}\theta,
\end{align*}
and 
\[
\|\X v\|^{2}=\|xv^{T}+vx^{T}\|_{F}^{2}=2\|x\|^{2}\|v\|^{2}+2(x^{T}v)^{2}\ge2\|x\|^{2}\|v\|^{2}.
\]
Finally, dividing by $\|v\|^{2}$ yields the desired bound.
\end{proof}

\subsection{Proof of Theorem~\ref{thm:foc}}

The problem of finding the best-conditioned $\H$ satisfying $\L(\H)=0$
is the following primal-dual LMI pair
\begin{align}
\underset{\H,\eta}{\text{maximize }} & \eta & \underset{y,U_{1},U_{2}}{\text{minimize }} & \tr(U_{2})\label{eq:foclmi}\\
\text{subject to } & \L(\H)=0, & \text{subject to } & \L^{T}(y)=U_{1}-U_{2},\nonumber \\
 & \eta I\preceq\H\preceq I. &  & \tr(U_{1})=1,\quad U_{1},U_{2}\succeq0,\nonumber 
\end{align}
where $\L^{T}$ is the adjoint operator to $\L$ in (\ref{eq:Ldef}).
Slater's condition is trivially satisfied by the dual: $y=0$ and
$U_{1}=U_{2}=\nu^{-1}I$ with $\nu=\frac{1}{2}n(n+1)$ is a strictly
feasible point. Hence, strong duality holds, meaning that the two
objectives coincide with $\tr(U_{2}^{\star})=\eta^{\star}$ at optimality,
so we implicitly solve the primal by solving the dual. 

The mechanics of the dual problem become more obvious if we first
optimize over $U_{1}$ and $U_{2}$ and the length of $y$. Applying
Lemma~\ref{lem:dual_prob} yields 
\begin{equation}
\underset{y}{\text{minimize }}\frac{\sum_{i=1}^{n}(-\lambda_{i}(\L^{T}(y))_{+}}{\sum_{i=1}^{n}(+\lambda_{i}(\L^{T}(y))_{+}}\qquad\text{ where }\qquad(\alpha)_{+}\equiv\begin{cases}
\alpha & \alpha\ge0\\
0 & \alpha<0
\end{cases}.\label{eq:dual_foc}
\end{equation}
The goal of this latter problem is to find a vector $y$ that maximizes
the sum of the positive eigenvalues of $\L^{T}(y)$, while minimizing
the (absolute) sum of the negative eigenvalues. In Lemma~\ref{lem:eig_bound},
we prove that $\L^{T}(y)$ has exactly one positive eigenvalue and
one negative eigenvalue, and their values in the rank-1 case are closely
related to the angle $\phi$ between $x$ and $z$. Substituting this
into (\ref{eq:dual_foc}) yields an unconstrained minimization
\[
\underset{y}{\text{minimize }}\frac{1-\cos\theta_{y}}{1+\cos\theta_{y}}\qquad\text{ where }\qquad\cos\theta_{y}=\frac{e^{T}\X y}{\|e\|\|\X y\|}.
\]
In turn, Lemma~\ref{lem:angle_bound} yields $\max_{y}\cos\theta_{y}=\cos\theta=\sqrt{1-\sin^{2}\theta}$
where $\sin\theta\equiv\zeta$ in the statement of Theorem~\ref{thm:foc}. 

\subsection{Proof of Theorem~\ref{thm:soc}}

We show that $\H_{\tau}\equiv\tau P_{e\perp}+\H_{0}$ with some $\tau\ge0$
is a feasible point for (\ref{eq:lmiopt}) with a small condition
number. Here, $P_{e\perp}=I-ee^{T}/\|e\|^{2}$ is the projection onto
the kernel of $e$, and $\H_{0}\preceq I$ is the best-conditioned
$\H\succeq0$ satisfying $\L(\H)=0$ from Theorem~\ref{thm:foc}.
Observe that $\cond(\H_{\tau})=(1+\tau)\cdot\cond(\H_{0})$.

Let us find the smallest $\tau\ge0$ to guarantee that $\L(\H_{\tau})=0$
and $\M(\H_{\tau})\succeq\mu I$, for some choice of $\mu>0$. Note
that $\L(\H_{\tau})=0$ is satisfied by construction, because $\L(\H_{\tau})=\tau\L(P_{e\perp})+(1-\tau)\L(\H_{0})$,
and $\L(\H_{0})=0$ by hypothesis while $\L(P_{e\perp})=2\X^{T}(P_{e\perp})e=0$.
Hence, our only difficulty is finding the smallest $0\le\tau<1$ such
that 
\[
\M(\H_{\tau})=2\mat(\H_{0}e)+\tau\X^{T}P_{e\perp}\X^{T}+\X^{T}\H_{0}\X^{T}\succ0.
\]
In Lemma~\ref{lem:soc}, we prove the following two inequalities
\[
\|\mat(\H_{0}e)\|\le\sqrt{1+\rho^{4}}\|z\|^{2},\qquad\lambda_{\min}(\X^{T}P_{e\perp}\X)\ge2\|x\|^{2}\zeta^{2}
\]
Hence, $\M(\H_{\tau})\succeq\mu I$ with $\mu=\sqrt{1+\rho^{4}}\|z\|^{2}\ge\|z\|^{2}$
is guaranteed if we set
\[
\frac{4\|\mat(\H_{0}e)\|}{\lambda_{\min}(\X^{T}P_{e\perp}\X)}\le\frac{4\sqrt{1+\rho^{4}}\|z\|^{2}}{2\|x\|^{2}\zeta^{2}}=\frac{2\sqrt{\rho^{2}+\rho^{-2}}}{\zeta^{2}}=\tau.
\]
Rescaling $\H_{\tau}$ by $1/(1+\tau)$ completes the proof for the
feasibility statement.

Finally, to derive the RIP constant bound $\delta_{2r}\le(\tau+\sqrt{1-\zeta^{2}})/(\tau+1)$,
write $\delta\equiv\sqrt{1-\zeta^{2}}$ and note that we have
\[
\eta_{2}^{\star}\ge\frac{1-\delta}{1+\tau}\cdot\frac{1}{1+\delta}=\left(1-\frac{\tau+\delta}{1+\tau}\right)\cdot\frac{1}{1+\delta}\ge\frac{1-(\tau+\delta)/(\tau+1)}{1+(\tau+\delta)/(\tau+1)},
\]
where the last bound is due to the fact that $(\tau+\delta)/(\tau+1)\ge\delta$
holds for all $\tau,\delta\ge0$. Multiplying through by $1+(\tau+\delta)/(\tau+1)$
yields
\[
\left(1-\frac{\tau+\delta}{\tau+1}\right)\|X\|_{F}^{2}\le\left(1+\frac{\tau+\delta}{\tau+1}\right)\|\AA(X)\|_{F}^{2}\le\left(1+\frac{\tau+\delta}{\tau+1}\right)\|X\|_{F}^{2}
\]
for all $X$. 
\end{document}